\documentclass[letterpaper]{article} % DO NOT CHANGE THIS
\usepackage{aaai24}  % DO NOT CHANGE THIS
\usepackage{times}  % DO NOT CHANGE THIS
\usepackage{helvet}  % DO NOT CHANGE THIS
\usepackage{courier}  % DO NOT CHANGE THIS
\usepackage[hyphens]{url}  % DO NOT CHANGE THIS
\usepackage{graphicx} % DO NOT CHANGE THIS
\usepackage{natbib}  % DO NOT CHANGE THIS AND DO NOT ADD ANY OPTIONS TO IT
\usepackage{caption} % DO NOT CHANGE THIS AND DO NOT ADD ANY OPTIONS TO IT
\usepackage{multirow}
\usepackage{theoremref}
\usepackage{xcolor}
\usepackage{multirow}
\usepackage{diagbox}
\usepackage{threeparttable}
\usepackage{verbatim}
\usepackage{booktabs}
\usepackage{subcaption}

\usepackage{algorithm}
\usepackage{algorithmic}
\usepackage{amsmath,amssymb,amsfonts}
\newtheorem{definition}{Definition}

\newtheorem{assumption}{Assumption}

\newtheorem{theorem}{Theorem}
\newtheorem{proof}{Proof}
\newtheorem{lemma}{Lemma}

\newcommand\ci{\perp\!\!\!\perp}

\makeatletter
\newcommand*{\indep}{%
	\mathbin{%
		\mathpalette{\@indep}{}%
	}%
}
\newcommand*{\nindep}{%
	\mathbin{%                   % The final symbol is a binary math operator
		\mathpalette{\@indep}{\not}% \mathpalette helps for the adaptation
	}%
}
\newcommand*{\@indep}[2]{%
	\sbox0{$#1\perp\m@th$}%        box 0 contains \perp symbol
	\sbox2{$#1=$}%                 box 2 for the height of =
	\sbox4{$#1\vcenter{}$}%        box 4 for the height of the math axis
	\rlap{\copy0}%                 first \perp
	\dimen@=\dimexpr\ht2-\ht4-.2pt\relax
	\kern\dimen@
	{#2}%
	\kern\dimen@
	\copy0 %                       second \perp
}
\usepackage{newfloat}
\usepackage{listings}
\DeclareCaptionStyle{ruled}{labelfont=normalfont,labelsep=colon,strut=off} % DO NOT CHANGE THIS
\floatstyle{ruled}
\newfloat{listing}{tb}{lst}{}
\floatname{listing}{Listing}
\title{Instrumental Variable Estimation for Causal Inference in Longitudinal Data with Time-Dependent Latent Confounders}
\author{ 	%Authors
	Debo~Cheng\equalcontrib, Ziqi~Xu\equalcontrib, Jiuyong~Li, Lin~Liu, Jixue~Liu, Wentao Gao and~Thuc~Duy~Le}
\affiliations{
	UniSA STEM, University of South Australia, Adelaide, SA, 5095, Australia
	 \{debo.cheng,jiuyong.li,lin.liu,jixue.liu,thuc.le\}@unisa.edu.au,\{ziqi.xu,wentao.gao\}@mymail.unisa.edu.au
	 
}

\usepackage{bibentry}
\begin{document}

\maketitle

\begin{abstract}
	Causal inference from longitudinal observational data is a challenging problem due to the difficulty in correctly identifying the time-dependent confounders, especially in the presence of latent time-dependent confounders. Instrumental variable (IV) is a powerful tool for addressing the latent confounders issue, but the traditional IV technique cannot deal with latent time-dependent confounders in longitudinal studies. In this work, we propose a novel Time-dependent Instrumental Factor Model (TIFM) for time-varying causal effect estimation from data with latent time-dependent confounders. At each time-step, the proposed TIFM method employs the Recurrent Neural Network (RNN) architecture to infer latent IV, and then uses the inferred latent IV factor for addressing the confounding bias caused by the latent time-dependent confounders. We provide a theoretical analysis for the proposed TIFM method regarding causal effect estimation in longitudinal data. Extensive evaluation with synthetic datasets demonstrates the effectiveness of TIFM in addressing causal effect estimation over time. We further apply TIFM to a climate dataset to showcase the potential of the proposed method in tackling real-world problems.
\end{abstract}

\section{Introduction}
Causal effect estimation based on observational data plays a crucial role in understanding the underlying causal mechanism of a system in various areas, e.g., in epidemiology, econometrics, clinical decision-making, and climate research~\cite{spirtes2000causation, imbens2015causal, runge2019inferring}. The confounding bias, causal by confounders that affect both the treatment variable and outcome variable, is a challenging problem for obtaining reliable causal effects using observational data, especially in the presence of latent confounders~\cite{hernan2006instruments}. 

Instrumental variable (IV) is a well-known tool for mitigating the spurious associations caused by latent confounders~\cite{martens2006instrumental}. The IV approach strongly relies on a predefined IV, which is required to be independent of all latent confounders and have a causal effect on the outcome only through its direct effect on the treatment.  Some IV-based methods have been developed to address latent confounding in causal effect estimation using observational data in the static setting~\cite{bowden1990instrumental, hartford2017deep,wang2022estimating}. However, only a few IV-based methods are available for the longitudinal setting~\cite{martinussen2017instrumental}. Estimating causal effects over time provides the benefits of analysing changes, interventions, actions, and relationships within the same subjects across different time-steps in many fields, e.g., education, medicine and clinical research. It also provides insights for making effective decisions over time~\cite{bica2019estimating}. 

The primary challenge of estimating causal effects from longitudinal studies is the time-dependent confounders which influence both the time-dependent treatment and the time-dependent outcome of interest, especially those latent time-dependent confounders~\cite{cui2023instrumental, michael2023instrumental}. For example, consider a study that investigates the effect of an `education intervention' (the treatment), such as personalised tutoring, on `academic performance' (the outcome) over time. In this scenario, certain factors like students' pre-existing knowledge levels and study habits, which change over time, might not be directly measured during the study, and are latent time-dependent confounders.

When all time-dependent confounders are measured, adjusting for these confounders using methods like inverse probability of treatment, can be utilised to obtain unbiased estimation of the causal effect of the time-dependent treatment~\cite{ali2016methodological}. In cases when there are latent time-dependent confounders, an IV approach with a valid time-dependent IV can be employed to recover the unbiased causal effect of time-dependent treatment from longitudinal data~\cite{robins2000marginalbook}. The treatment, outcome and IV can be time-varying due to the dynamic and evolving nature~\cite{robins1986new}. Following the example mentioned above, the 'government funding in education' can potentially serve as a time-dependent IV for estimating the causal effect of time-dependent education intervention on time-dependent academic performance. A valid time-dependent IV is crucial for the soundness of the IV approach in longitudinal data, but finding a valid time-dependent IV poses a major challenge in applications.

Data-driven methods for discovering time-dependent IV directly from longitudinal data are needed. Many works using the IV approach in longitudinal studies
need a given valid time-dependent IV~\cite{cui2023instrumental}). However, in practice, it is often impossible to know such an IV in advance, which makes it impossible to estimate the causal effect in this case. On the other hand, although there might not be a variable satisfying the requirements for a time-dependent IV directly, the information of time-dependent IV could be inferred from some observed variables.

In this work, we propose a novel sequential architecture based on a recurrent neural network (RNN), called the Time-dependent Instrumental Factor Model (TIFM) for learning a substitute of a time-dependent IV from the historical data of the covariates in longitudinal data with time-dependent latent confounders. Then the learned substitute in each sequence is used as a time-dependent IV in the Two-Stage Least Squares (TSLS) regression~\cite{angrist1995two} to mitigate the confounding bias caused by latent time-dependent confounders and obtain unbiased sequential causal effect estimations.   

We summarise the contributions of our paper as follows.
\begin{itemize}
	\item We study a crucial problem in a longitudinal study and propose to use a substitute of latent time-dependent IV for estimating causal effect in the presence of latent time-dependent confounders. We theoretically analyse the soundness of the substitute of latent time-dependent IV in causal effect estimation from longitudinal data with latent time-dependent confounders.
	\item  We propose a novel Time-dependent Instrumental Factor Model, TIFM, to learn the substitute of latent time-dependent IV from data. To the best of our knowledge, this is the first method which learns a substitute of time-dependent IV from longitudinal data directly.  
	\item Extensive experiments on synthetic datasets and a real-world climate dataset demonstrate the effectiveness of TIFM in causal effect estimation from longitudinal data with latent time-dependent confounders. 
\end{itemize}

\section{Problem Setting}
Throughout the paper, we use uppercase letters to indicate variables and lowercase letters for their values. A bold-faced letter is used to represent a set of variables (in uppercase) or the corresponding values of a set of variables (in lowercase).

For an individual $i$, the data consists of time-dependent covariates $\bar{\mathbf{X}}_{t}^{(i)}=(\mathbf{X}_{1}, \dots, \mathbf{X}_{t})\in\mathcal{X}_{t}$, history treatment $\bar{W}_{t}^{(i)}=(W_1, \dots, W_t)\in\mathcal{W}_t$ and outcome $\bar{Y}_{t+1}^{(i)}=(Y_2, \dots, Y_{t+1})\in\mathcal{Y}_{t+1}$ for $t$ discrete time-steps. Furthermore, $\mathbf{X}_{t} = [\mathbf{X}_{t1}, \dots, \mathbf{X}_{tk}]\in\mathcal{X}_{t}$ be the $k$ covariates. Note that $\bar{Y}_{t+1}$ is a sink node in the full-time graph\footnote{The infinite directed acyclic graph (DAG) over time~\cite{peters2017elements,robins2000marginalbook}.}~\cite{mastakouri2021necessary}, i.e., there is not a descendant node of $\bar{Y}_{t+1}^{(i)}$. The number of time-steps $t$ ranges within $\{1, \ldots, T\}$ and is not random. In addition to the measured data, let $\bar{\mathbf{U}}_{t}^{(i)}=(U_1, \dots, U_t)\in\mathcal{U}_t$ be an unmeasured time-dependent variables, which affect  both $\bar{W}_{t}^{(i)}$ and $\bar{Y}_{t+1}^{(i)}$. Note that the superscript ($i$) for the specific individual will not be used unless explicitly mentioned. 

Let $Y_{t+1}(\bar{w})$ be the \emph{potential outcomes} relative to each possible value of treatment $\bar{w}$. These \emph{potential outcomes} are not generally measured, and their correlation with the measured data is based on the \emph{consistency} assumption.

\begin{assumption}[Consistency]
	\label{ass:consis}
	For a given individual, if $\bar{W}_{\ge t}=\bar{w}_{\ge t}$, then $Y(\bar{w}_{\ge t}) = Y$, i.e., the potential outcomes on $\bar{w}_{\ge t}$ is the same as the factual outcome $Y$. 
\end{assumption}

We aim to recover the average causal effects over time in the presence of latent time-dependent confounders $\bar{\mathbf{U}}_t$ between $\bar{W}_t$ and $\bar{Y}_{t+1}$, where $\bar{\mathbf{U}}_t$ affects both $\bar{W}_t$ and $\bar{Y}_{t+1}$. The latent Sequential Randomisation Assumption (SRA)~\cite{robins2000marginalbook} holds if $\bar{\mathbf{U}}_t$ is measured:
\begin{assumption}[Latent SRA]
	\label{ass:LSRA}
	$Y(\bar{w})\indep W_t\mid \bar{W}_{t-1}=\bar{w}_{t-1},\bar{\mathbf{X}}_t, \bar{\mathbf{U}}_t$.
\end{assumption}

Note that the latent confounders $\bar{\mathbf{U}}_t$ affects both $\bar{W}_t$ and $\bar{Y}_{t+1}$, resulting in non-identifiable causal effects of $\bar{W}_t$ on $\bar{Y}_{t+1}$~\cite{pearl2009causality,hernan2020causal}. However, traditional IV methods~\cite{hartford2017deep, cheng2023causal} for the static setting cannot be used to obtain unbiased estimation of the causal effects of $\bar{W}_t$ on $\bar{Y}_{t+1}$ over time. In this work, we aim to utilise the ``time-dependent IV'' to address the effect of latent time-dependent confounders $\bar{\mathbf{U}}_t$ in causal effect estimation over time.

We assume that there exists a latent time-dependent IV $\bar{S}_t$ caused by a set of time-dependent covariates $\bar{\mathbf{X}}_t$, satisfying the following longitudinal generalisation of the standard IV assumptions~\cite{robins1991correcting,cui2023instrumental} as described below, for $1\leq t\leq T$, and $\bar{w}_t\in\mathcal{W}^T$: 
\begin{assumption}[IV relevance]
	\label{ass:ivre}
	$\mathbb{E}(W_t\mid \bar{W}_{t-1}, \bar{\mathbf{X}}_t, \bar{S}_t)\neq\mathbb{E}(W_t\mid \bar{W}_{t-1}, \bar{\mathbf{X}}_t, \bar{S}_{t-1})$.
\end{assumption}
The assumption says that the IV $\bar{S}_t$ should be associated with $\bar{W}_t$ conditioning on the history data. 

\begin{assumption}[IV-outcome independence]
	\label{ass:ivout}
	$\bar{S}_t\indep(Y_{t+1} (\bar{w}), \mathbf{X}_{t+1}, \mathbf{U}_{t+1})\mid \bar{W}_t =\bar{w}_t, \bar{\mathbf{X}}_t,  \bar{\mathbf{U}}_t$.
\end{assumption}
This assumption essentially indicates that there is not a direct causal effect of $\bar{S}_t$ on $\mathbf{X}_{t+1}$, and $\mathbf{U}_{t+1}$ and $Y_{t+1}$ will be identified conditioning on history data if one set $\bar{W}_t=\bar{w}_t$~\cite{cui2023instrumental}.

\begin{assumption}[IV–unmeasured confounder independence]
	\label{ass:ivuncon}
	$S_t\indep\bar{\mathbf{U}}\mid \bar{W}_{t-1}, \bar{\mathbf{X}}_t,  \bar{S}_{t-1}$.
\end{assumption}

\begin{figure}[t]
	\centering
	\includegraphics[scale=0.439]{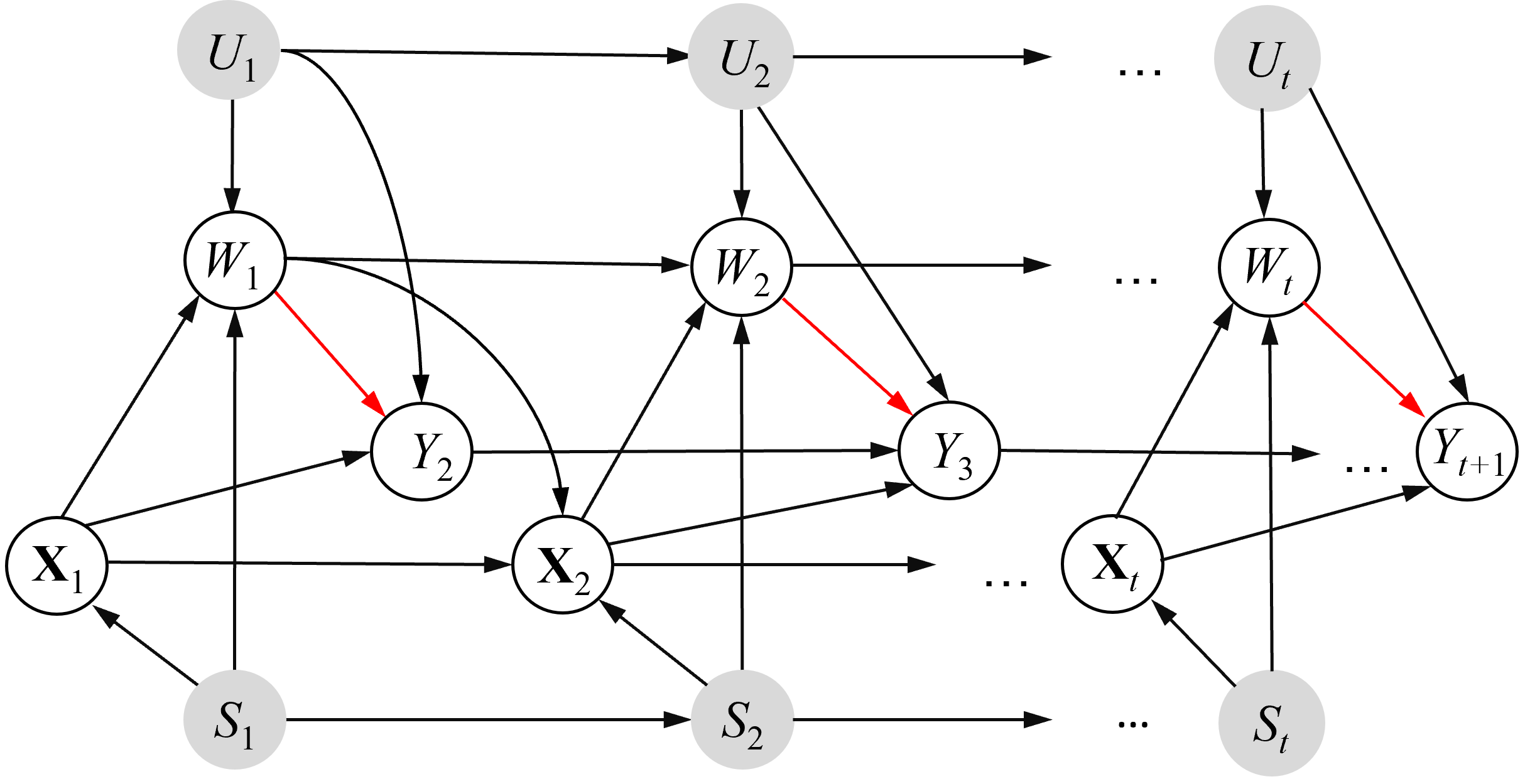}
	\caption{A causal DAG illustrating our problem setting and time-dependent IV, where we have latent time-dependent confounders $\bar{U}_t\in\bar{\mathbf{U}}_t$ (indicated by shaded circles), and the time-dependent confounders $\bar{\mathbf{X}}_t$, and a latent time-dependent IV $\bar{S}_t$ (indicated by shaded circles), satisfying Assumptions \ref{ass:ivre}, \ref{ass:ivout} and \ref{ass:ivuncon}. We aim to estimate the causal effects of $\bar{W}_t$ on $\bar{Y}_{t+1}$ over time by inferring and using the substitute of $\bar{S}_t$ from $\bar{\mathbf{X}}_t$ and $\bar{W}_t$ for estimating the causal effects of  $\bar{W}_t$ on $\bar{Y}_{t+1}$ over time in the presence of latent time-dependent confounders $\bar{\mathbf{U}}_t$.}
	\label{fig:ivdag}
\end{figure}

Assumption~\ref{ass:ivuncon} says that $S_t$ and $\bar{\mathbf{U}}$ are independent conditioning on history data and ${\mathbf{X}}_t$.  These three assumptions are utilised to formalise the temporal relationships within the data. Fig.~\ref{fig:ivdag} presents an illustrative causal DAG, offering a clear interpretation of the intricate relationships among  $\bar{S}_t$, $\bar{\mathbf{X}}_t$, $\bar{\mathbf{U}}_t$, $\bar{W}_t$ and $\bar{Y}_{t+1}$.

It is worth mentioning that both assumptions \ref{ass:ivout} and \ref{ass:ivuncon} are not testable from longitudinal data in the presence of $\bar{\mathbf{U}}_t$. Moreover, most IV-based methods in longitudinal studies require a known time-dependent IV, which may not be available in numerous real-world applications. The challenge of identifying or searching for a valid time-dependent IV remains an unresolved problem in causal inference within a temporal context. To offer an alternative approach to this open problem in causal inference with longitudinal data, in this paper,  we focus on exploring how to learn time-dependent IV with minimal reliance on domain knowledge.

\section{The Proposed TIFM Method}
In this work, we aim to learn a substitute of the latent time-dependent IV $\bar{S}_t$ from the history data for estimating the causal effects of $\bar{W}_t$ on $\bar{Y}_{t+1}$ in the presence of latent time-dependent confounders. Then, we apply the inferred substitute of the latent time-dependent IV within the Two-Stage Least Squares regression (TSLS) to estimate the average causal effects over time in the presence of latent time-dependent confounders.  

\subsection{Objective Formulation}
Causal inference requires certain assumptions to draw valid causal conclusions from observational data, particularly when latent time-dependent confounders are present. In this work, we focus on the assumed setting shown in Fig.~\ref{fig:ivdag} to infer the time-dependent IV for estimating the causal effects of $\bar{W}_t$ on $\bar{Y}_{t+1}$ over time in the presence of latent time-dependent confounders $\bar{\mathbf{U}}_t$. 

Specifically, our objective is to infer a sequence of latent factors $\bar{\mathbf{L}}_t =(L_1, \dots, L_t)\in \mathcal{L}_t$ as a substitute of $\bar{S}_t$ from history data using deep learning techniques. Then, the inferred substitute $\bar{\mathbf{L}}_t$ serves as a time-dependent IV for causal inference from longitudinal data with latent time-dependent confounders $\bar{\mathbf{U}}_t$. 

Recently, Bica et al.~\cite{bica2020time} have built a factor model \emph{over time} to infer latent variables render the assigned multiple treatments conditionally independent in the estimation of treatment responses over time. Different from the work~\cite{bica2020time}, we consider a single treatment $W_t$, rather than multiple treatments at time $t$. Note that a single treatment is a much more difficult problem setting than the multiple treatments setting, in the sense that single treatment provides less information about the latent time-dependent confounders $\bar{\mathbf{U}}_t$. Similar to the work~\cite{wang2019blessings,bica2020time}, we also make the no latent single-cause confounders assumption: there are no single-cause confounders between $\bar{\mathbf{X}}_t$ and $\bar{W}_t$, i.e., no variable that affects just one of $\bar{\mathbf{X}}_t$ and $\bar{W}_t$ over time.

\subsection{Solution Outline}
Let $\bar{\mathbf{H}}_{t-1}=(\bar{W}_{t-1},\bar{\mathbf{X}}_{t-1}, \bar{\mathbf{L}}_{t-1})$. At time $t$, TIFM constructs the latent variable $\mathbf{l}_t=f(\bar{\mathbf{h}}_{t-1})$, where $\bar{\mathbf{h}}_{t-1}=(\bar{w}_{t-1}, \bar{\mathbf{x}}_{t-1}, \bar{\mathbf{l}}_{t-1})$ denotes the values of history data $\bar{\mathbf{H}}_{t-1}$. In our problem setting, $\bar{S}_t$ in the causal DAG in Fig.~\ref{fig:ivdag} is unobserved. The latent variable $\bar{\mathbf{L}}_t$ renders the marginal distribution of $\mathbf{X}_t$ as:
\begin{equation}
	\label{eq:001}
	p(x_{t1}, \dots, x_{tk}\mid \mathbf{l}_t) = \prod_{j=1}^{k} p(x_{tj}\mid\mathbf{l}_t)
\end{equation} 

To obtain the substitute $\bar{\mathbf{L}}_t$, the factor model of $\bar{\mathbf{X}}_t$ with joint distribution is built on $\bar{\mathbf{X}}_t$ and history data  $\bar{\mathbf{H}}_{t-1}$ as a latent variable model:
\begin{equation}
	\begin{aligned}
		p(\sigma_{1:k},\bar{\mathbf{x}}_T,\bar{\mathbf{l}}_T) &=  p(\sigma_{1:k})\times \\&  \prod_{t=1}^{T} (p(\mathbf{l}_t\mid \bar{\mathbf{h}}_{t-1}) \prod_{j=1}^{k} p({x}_{tj}\mid\mathbf{l}_t,\sigma_j))
	\end{aligned}
	\label{eq:factorM}
\end{equation}

\noindent where $\sigma_{1:k}$ are parameters. We fit Eq. (\ref{eq:factorM}) to capture the dependencies among the covariates caused by the latent time-dependent IV $\bar{S}_t$. We then infer $\bar{\mathbf{L}}_t$, which can be regarded as a substitute for $\bar{S}_t$. Thus, by leveraging the dependencies between multiple covariate measurements, the factor model enables us to recover a sequence of latent variables $\bar{\mathbf{L}}_t$ from history data. 

To provide a theoretical proof for Eq. (\ref{eq:factorM}), we first introduce the following Sequential Kallenberg Construction, a modified definition of the ``Kallenburg Construction'' in~\cite{kallenberg1997foundations,bica2020time}.

\begin{definition} [Sequential Kallenberg construction]
	At time-step $t$, the distribution of $\mathbf{X}_t = [\mathbf{X}_{t1}, \dots, \mathbf{X}_{tk}]$ follows a sequential Kallenberg construction through the random variables $\mathbf{l}_t = f(\bar{\mathbf{h}}_{t-1})$, provided that there exist measurable functions $f_{tj}: \mathcal{L} \times [0, 1] \to \mathcal{X}_t$ and random variables $M_{tj} \in [0, 1]$, where $j = 1, \dots, k$, satisfying the condition $\mathbf{X}_{tj} = f_{tj}(\mathbf{L}_t, M_{tj})$, with $M_{tj} \in [0, 1]$ jointly satisfying $(M_{t1}, \dots, M_{tk})\ci \bar{W}_t\mid \bar{\mathbf{X}}_{t-1},\bar{\mathbf{L}}_t$.
\end{definition}

We then present the following theorem, which guarantees that  $\bar{\mathbf{L}}_t$  can serve as a valid time-dependent instrument for unbiased causal effect estimation of $\bar{W}_t$ on $\bar{Y}_{t+1}$ using longitudinal data, even in the presence of latent time-dependent confounders $\bar{\mathbf{U}}_t$. 
\begin{theorem}
	\label{theo:01}
	If the distribution $p(\bar{\mathbf{x}}_T)$ can be represented using the factor model $p(\sigma_{1:k}, \bar{\mathbf{x}}_T, \bar{\mathbf{l}}_T)$, we can deduce that $\bar{S}_t$ is captured by the substitute $\bar{\mathbf{L}}_t$ which serves as a time-dependent IV.
\end{theorem}

Theorem~\ref{theo:01} guarantees the soundness of our proposed TIFM method.  $\bar{\mathbf{L}}_t$ is a substitute of $\bar{S}_t$ and plays the role of time-dependent IV in the estimation of the causal effect of $\bar{W}_t$ on $\bar{Y}_{t+1}$ from longitudinal data in the presence of latent time-dependent confounders $\bar{\mathbf{U}}_t$. In the next section, we introduce our implementation of TIFM method over time in practice. 

\subsection{Implementation}
In this section, we present an implementation for acquiring the substitute of time-dependent IV. 
Long Short-Term Memory (LSTM) is a specialised type of recurrent neural network (RNN) that is well-suited for tasks involving sequences and time-series data. 
Hence, we devise a LSTM architecture primed to capture the substitute of latent time-dependent IV from history data. Formally, the component dedicated to substitute IVs is outlined as follows:
\begin{equation}
	\begin{aligned}
		&\mathbf{L}_{1} = \text{LSTM}(\psi),\\
		&\mathbf{L}_{t} = \text{LSTM}(\mathbf{L}_{t-1},\mathbf{X}_{t-1}),
	\end{aligned}
\end{equation} 
\noindent where $\psi$ is the randomly initialised parameter for the initial step, which is subsequently trained in conjunction with the remaining parameters within the LSTM.

After training the LSTM model, we extract $\bar{\mathbf{L}}_t$ from the model and utilise it as the time-dependent IV for an IV-based causal effect estimation. It is essential to emphasise that our TIFM method imposes no restrictions on subsequent causal effect estimators, thereby allowing for the incorporation of well-known estimators like TSLS~\cite{angrist1995two}, DeepIV~\cite{hartford2017deep}, and Ortho.IV~\cite{syrgkanis2019machine}, as plug-ins within the overarching framework. The architecture of the proposed TIFM method is visually depicted in Fig.~\ref{pic:model}.

\begin{figure}[t]
	\centering
	\includegraphics[scale=0.35]{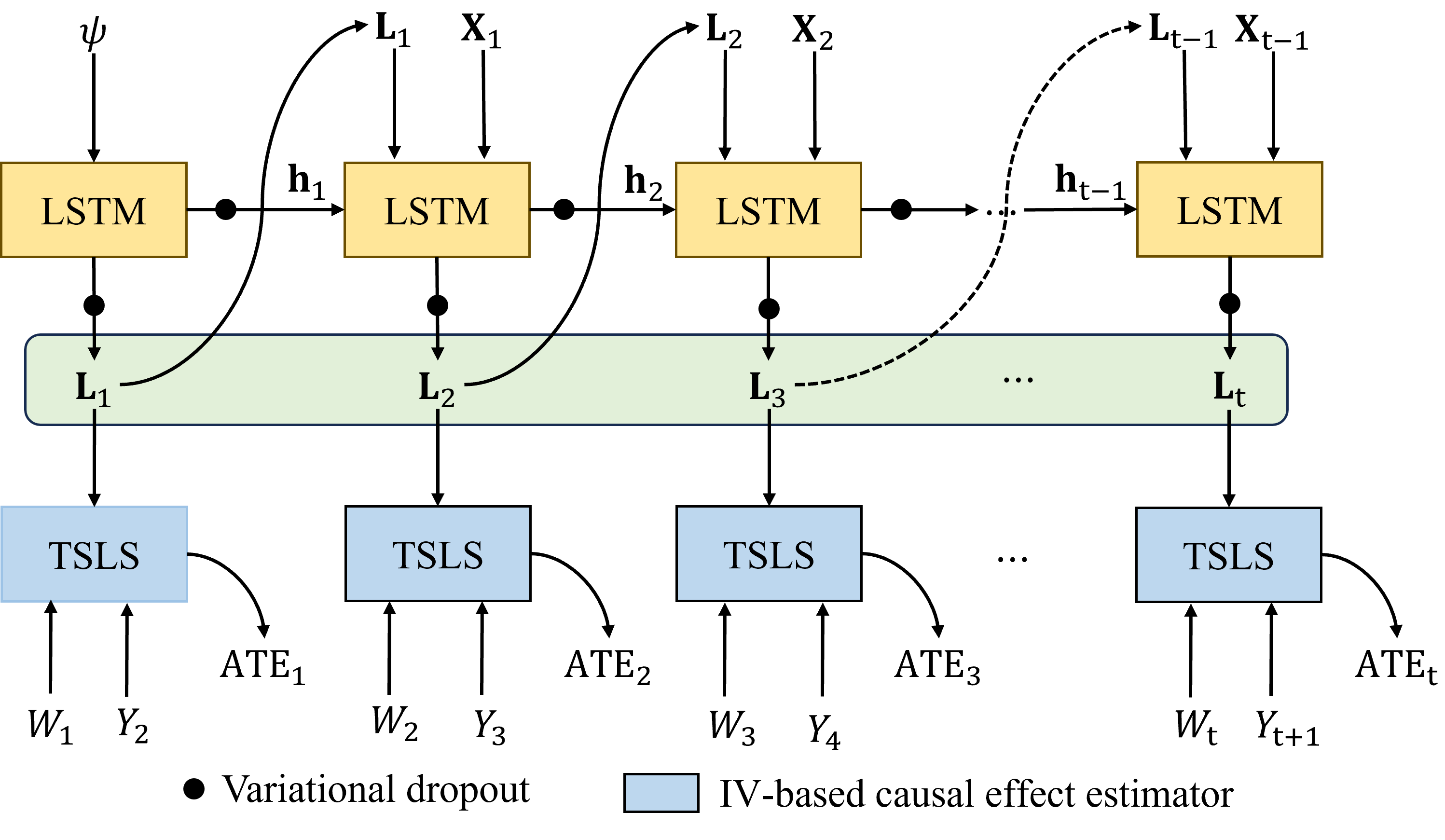}
	\caption{An overview on the architecture of the TIFM method. $\mathbf{L}_{t}$ is generated by LSTM as a function of the history state $\mathbf{h}_{t}$ and current input. The TSLS estimators in the blue rectangular can be replaced by any possible IV-based causal effect estimator.}
	\label{pic:model}
\end{figure}

\section{Experiments}
To validate the performance of the proposed TIFM method, we first conduct experiments on synthetic datasets, where we have access to the true time-dependent IVs, enabling precise computation of the actual causal effects over time. Following the methodologies outlined in~\cite{wang2019blessings, bica2020time}, we generate synthetic datasets to assess the performance of our TIFM. Comparative evaluations are drawn against the state-of-the-art estimators. Furthermore, sensitivity analyses on TIFM are carried out to glean insights into its resilience across diverse parameter configurations. To demonstrate the practical applicability of our TIFM, we further apply TIFM to a real-world climate dataset as a case study, verifying its effectiveness in real-world scenarios. 

\subsection{Experiment Setup}
To ensure fair comparisons, we perform simulation studies on several synthetic datasets. These datasets are generated by adhering to the data generation process outlined in the study by Bica et al.~\cite{bica2020time}. The specifics of our synthetic dataset generation procedure are available in the supplement due to space constraints. In this work, we generate the synthetic datasets with a variety range of sample sizes:  2k, 4k, 6k, and 8k. To avoid the bias brought by the data generation process, we repeatedly generate 30 datasets for each sample size. To induce time dependencies, we set $p = 1$ and $p = 3$ ($p$ is from $p$-order autoregressive processes, the detail can be found in the supplements.). The dimensionality of the covariates $\mathbf{X}$ and the latent confounders are set to 3, respectively. 

\begin{figure*}[t]
	\centering
	\includegraphics[scale=0.566]{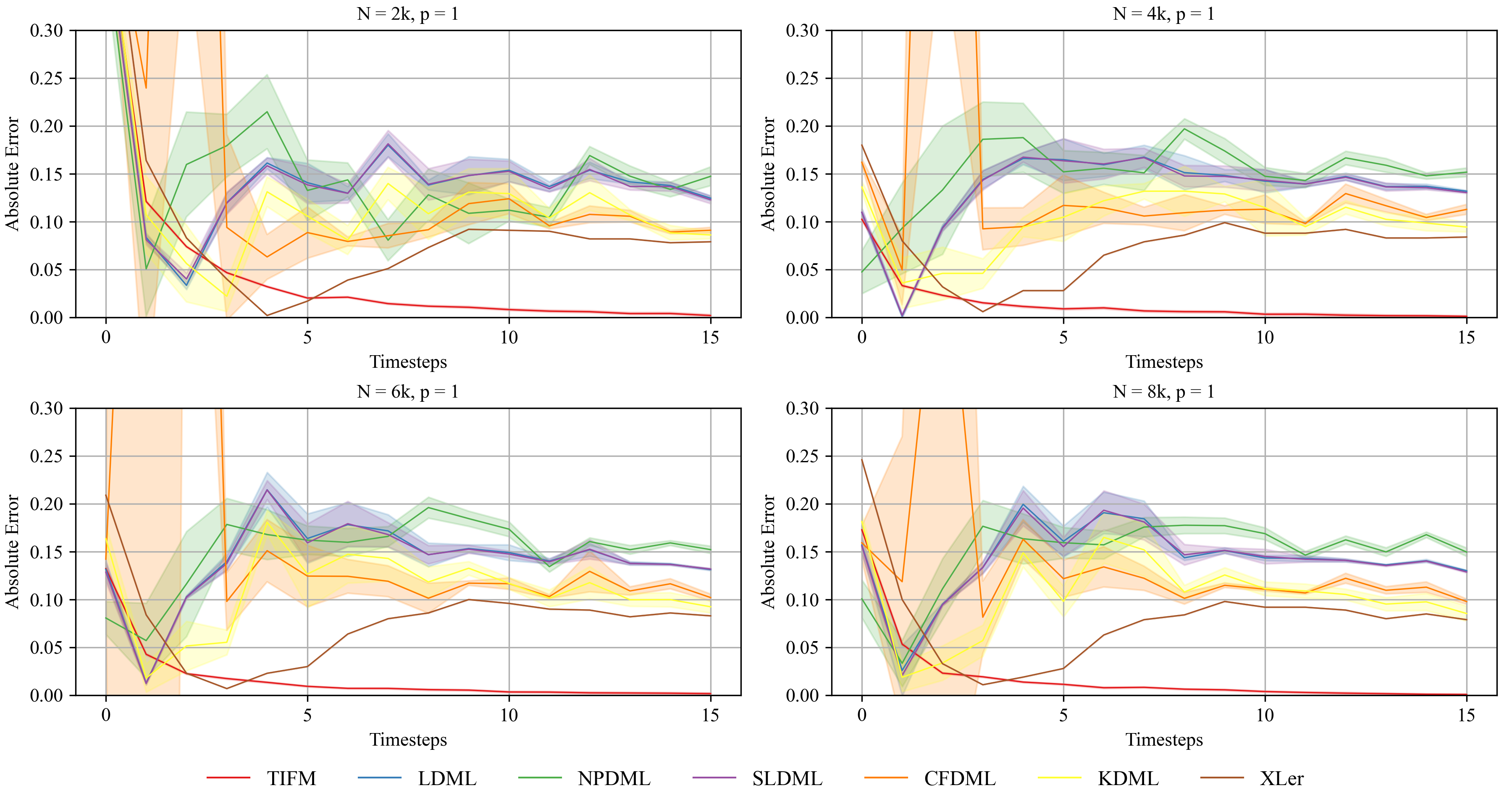}
	\caption{Results of the absolute errors across all estimators, accompanied by the mean plus standard deviation calculated over 30 synthetic datasets. Notably, our TIFM consistently shows a reduction in absolute error as the time-step increases. }
	\label{fig:res_simulate}
\end{figure*}

\subsubsection{Models for Comparison} We compare TIFM with multiple state-of-the-art causal effect estimators,  including: (1) LinearDML (LDML)~\cite{chernozhukov2018double}, which addresses the reverse causal metric bias by applying the cross-fitting strategy; (2) NonParamDML (NPDML)~\cite{chernozhukov2018double}, which is a non-parametric version of Double ML estimators that can have arbitrary final ML models; (3) SparseLinearDML (SLDML)~\cite{Semenova2023Inference}, for which loss function of the LinearDML estimator is modified by incorporating $L_1$ regularisation; (4) CausalForestDML (CFDML)~\cite{athey2019generalized}, which employs two random forests for causal estimations for predicting two potential outcomes respectively; (5) KernelDML (KDML)~\cite{nie2021quasi}, which combines dimensionality reduction techniques and kernel methods; (6) Mete-learner~\cite{kunzel2019metalearners}, specifically, the X-learner (XLer). The aforementioned machine learning-based estimators operate within a static setting and result in biased estimations when applied to longitudinal data, as they disregard the time dependencies between covariates.

\noindent\textbf{Notes}. IV-based estimators working in a static setting are not compared in our experiments since these methods assume a known IV explicitly included in the dataset, but in our problem setting, the IV is not measured. Some well-known double-robust models are used for comparison, but the experimental results show that these models produce a huge deviation in the longitudinal data, so we will not report them in the main text and provide the results in the supplement. We also note that there are some models that estimate causal effects over time, including Standard Marginal Structural Models~\cite{robins2000marginal,hernan2001marginal}, Recurrent Marginal Structural Networks~\cite{lim2018forecasting} and Time Series Deconfounder~\cite{bica2020time}. However, these models are completely different from the tasks we focus on, so they cannot be compared  as discussed in our related work. They produce predictions of potential outcomes, but our TIFM method focuses on the estimation of causal effects for each time-step. 

\subsubsection{Evaluation Criterion} For evaluating the performance of
TIFM and the comparison models, we use the absolute error $|\hat{\beta} - \beta|$ as the metric, where $\hat{\beta}$ is the estimated results and $\beta$ is the ground truth.

\subsubsection{Implementation Details} The implementations of comparison models are from the $Python$ package $econml$~\cite{econml}. We use TensorFlow~\cite{tensorflow2015} to implement our proposed TIFM method. We provide the implementation of our TIFM method and the parameter settings in the supplement.

\begin{table*}[t]
	\centering
	\caption{Experimental results for sensitivity analysis on the synthetic datasets with 5k.}
	\begin{tabular}{ccccccc}
		\toprule
		\multicolumn{2}{c}{ }          & time-step-1          & time-step-5        & time-step-10       & time-step-15       & time-step-20       \\ \midrule
		\multirow{2}{*}{$p$ = 1} & Baseline & 0.227$\pm$0.000 & 0.283$\pm$0.000 & 0.251$\pm$0.000 & 0.231$\pm$0.000 & 0.209$\pm$0.000 \\
		& TIFM     & 0.173$\pm$0.001   & 0.013$\pm$0.001 & 0.005$\pm$0.001 & 0.001$\pm$0.001 & 0.005$\pm$0.001 \\ \midrule
		\multirow{2}{*}{$p$ = 3} & Baseline & 0.280$\pm$0.000   & 0.323$\pm$0.000 & 0.335$\pm$0.000 & 0.311$\pm$0.000 & 0.272$\pm$0.000 \\
		& TIFM     & 0.287$\pm$0.002   & 0.023$\pm$0.001 & 0.013$\pm$0.000 & 0.031$\pm$0.001 & 0.058$\pm$0.002 \\ \bottomrule
	\end{tabular}
	\label{tab:compare}
\end{table*}

\subsection{Performance Evaluation}
The absolute errors across all estimators on 30 synthetic datasets are visualised in Fig.~\ref{fig:res_simulate}.
From Fig.~\ref{fig:res_simulate}, we have that the proposed TIFM method achieves the lowest absolute error compared with other methods. Although the comparison methods can adjust for covariates to enable unbiased causal inference in a static setting, they still struggle to handle latent confounders and time-dependent onfounders, resulting in some large estimation errors.

Our proposed TIFM method learns a suitable substitute for time-dependent IV from history data, followed by using an IV-based estimator to estimate causal effects at each time-step. By doing so, our TIFM effectively addresses bias arising from time dependence, while the IV estimator employed by TIFM tackles bias from time-dependent confounders and latent confounders. As a result, our TIFM method consistently outperforms the others, especially beyond the fifth time-step.

Please note that additional comprehensive results for various sample sizes and settings of $p = 3$ are provided in the supplement. The success of our proposed TIFM method in achieving lower absolute errors demonstrates its superior performance in handling complex confounding factors and time dependencies. 

\subsection{Sensitivity Analysis}
In this section, we conduct a sensitivity analysis on the parameter $p$ (time dependencies), exploring its impact on the performance of our proposed TIFM. Due to page limitations, we maintain a fixed sample size of 5k while varying the value of $p$ to compare the baseline\footnote{The $\hat{\beta}$ for each time step is calculated without any adjustment.} with TIFM. For detailed results, please refer to the supplement.

Table~\ref{tab:compare} presents the comparative results under different time-steps, and we observe no significant changes in the outcomes. Specifically, the baseline model exhibits a considerable absolute error when compared to TIFM. Notably, as we manipulate the value of $p$, we notice that increasing $p$ leads to a decline in performance. This trend is consistent with our expectations, as higher values of $p$ imply increased complexity of time dependence, resulting in deteriorated estimation performance for static models. TIFM is also affected to some extent by changes in $p$. As $p$ increases, we observe a slight increase in the absolute error; however, overall, our TIFM continues to demonstrate superior performance compared to the baseline.

\subsection{Case Study: NCEP-NCAR Reanalysis 1 Dataset}
The dataset used in this study is sourced from the National Centers for Environmental Prediction (NCEP) and the National Center for Atmospheric Research (NCAR). It is a comprehensive global climate dataset widely employed in atmospheric research~\cite{kalnay1996ncep}. The dataset encompasses a diverse array of variables, such as precipitation rate (prate), pressure level (pres), air temperature (air), skin temperature (skt), downward short-wave radiation flux (dswrf), clear-sky upward solar flux (csusf), clear-sky downward longwave flux (csdlf), cloud forcing net longwave flux (cfnlf), wind speed (wspd), minimum temperature (tmin), and seasonal categories (season).

For the analysis conducted in this research, we focus on specific regions, namely the United States (403 data points) and Europe (418 data points). Subsequently, we select cfnlf, wspd, and skt as the treatments under investigation, aiming to estimate their respective causal effects on prate (precipitation rate) at each time-step. The total time-steps are set to 120, encompassing monthly intervals from 2013 to 2022.

We apply our TIFM method to the two different regions and present the results for the United States in Fig.~\ref{pic:usa}. The analysis reveals that skin temperature  (skt) has the most significant impact on rainfall in the United States. This finding aligns with domain experts' knowledge, as the diverse climates in the United States make surface temperature a crucial driver of rainfall, especially in areas where high temperatures can trigger strong convective activity~\cite{seeleyromps2020}.

\begin{figure}[t]
	\centering
	\includegraphics[scale=0.66]{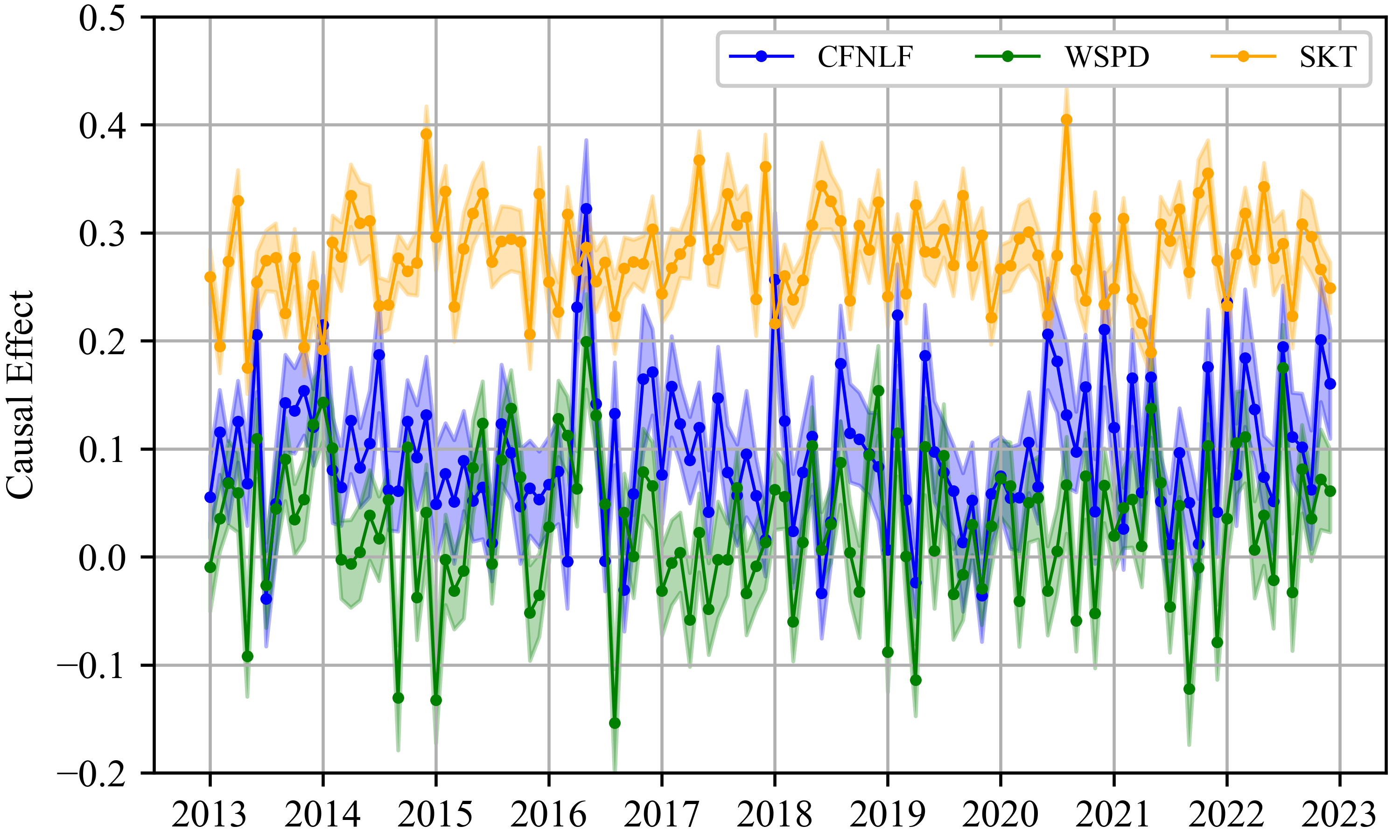}
	\caption{The estimated causal effects over time by our proposed TIFM method on the data from United States.}
	\label{pic:usa}
\end{figure}

\begin{figure}[t]
	\centering
	\includegraphics[scale=0.66]{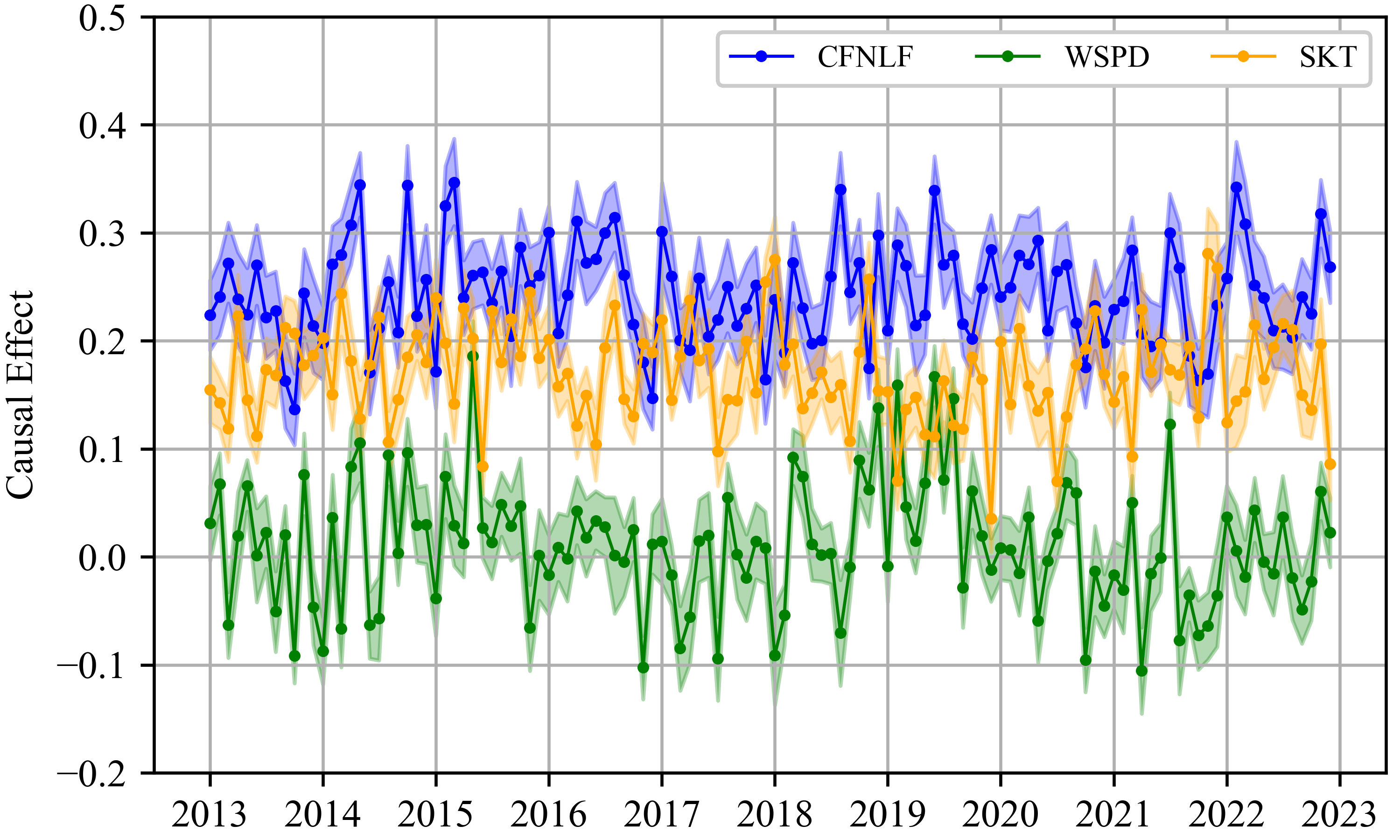}
	\caption{The estimated causal effects over time by our proposed TIFM method on the data from Europe.}
	\label{pic:europe}
\end{figure}

Fig.~\ref{pic:europe} displays the results for Europe. In this region, wind speed does not directly cause rainfall; rather, it influences cloud formation and distribution by transporting water vapour. When these clouds contain sufficiently large water droplets, precipitation occurs, leading to rainfall~\cite{haylock2008european, uppala2005era, zolina2013changes}. The result in Fig.~\ref{pic:europe} is consistent with the domain knowledge and the result indicates that cloud forcing net longwave flux (cfnlf) has the most significant causal effect on precipitation (prate) in Europe.

Overall, the impact of these three causal factors on rainfall varies across different regions, and their exact magnitude of influence may be influenced by other latent factors, such as topography and vegetation cover, which are challenging to directly measure. However, our proposed TIFM method effectively captures the interactions between these latent factors and historical records. Despite the assumption that historical information serves as a proxy for these latent factors, our experimental results show the applicability of the TIFM method in real-world problems.

\section{Related Work}
\label{sec:relatedworks}
Most prior IV-based methods have primarily concentrated on causal effect estimation from observational data within a static setting~\cite{imbens2014instrumental,athey2019generalized,cheng2022data}. There have been relatively fewer methods dealing with longitudinal data. In this study, we delve into methods for estimating causal effects over time while considering the influence of latent time-dependent confounders, by employing time-dependent IV. In this section, we review the work closely relative to our TIFM method, including IV-based methods in static settings, treatment effect estimations over time and time-dependent IV-based methods for causal effect estimation.   

\noindent \textbf{ IV-based methods in static setting}. Instrumental variables are a powerful approach for mitigating the confounding bias arising from latent confounders in causal effect estimation \cite{hernan2006instruments,imbens2014instrumental}. For example, Hartford et al. \cite{hartford2017deep} introduced an innovative IV method called DeepIV, which utilises deep ensembles to estimate causal effects in nonlinear scenarios. Athey et al. \cite{athey2019generalized} devised an IV method based on random forest regression. Cheng et al. \cite{cheng2023causal} introduced a novel method which utilises a deep generative model to construct and disentangle the representations of conditional IV and its corresponding conditioning set using observational data. In contrast to these IV-based methods designed for a static setting, our focus centres on the acquisition of a valid time-dependent IV using longitudinal data for the estimation of causal effects over time, particularly in the presence of latent time-dependent confounders.

\noindent \textbf{Treatment effect estimations over time}.  The realm of causal effect estimation over time has seen prominent contributions from the epidemiology community, encompassing techniques such as $g$-computation, structural nested mean models (SNMM), and marginal structural mean models (MSMM) \cite{robins1986new,robins1997causal,robins2000marginal,robins2000marginalbook}. These methods commonly resort to logistic or linear regression models for prediction, rendering them ill-suited for addressing intricate time-dependent relationships. Recently, Lim et al. \cite{lim2018forecasting} introduced a novel concept, the recurrent marginal structural networks, to predict the evolution of treatment responses over time. Similarly, Bica et al. \cite{bica2019estimating} presented a novel solution: the counterfactual recurrent network (CRN), a sequence-to-sequence model. CRN constructs a treatment-invariant representation that empowers the prediction of counterfactual scenarios. Moreover, Bica et al.~\cite{bica2020time} proposed a time series deconfounder (TSD) employing a recurrent neural network (RNN) architecture for multiple treatments causal effect estimation, further expanding the CRN for analysing the complexities of time series data.

Recently,  Melnychuk et. al.~\cite{melnychuk2022causal}  proposed a novel Causal Transformer model (CT) for estimating counterfactual outcomes from observational data, particularly suited for capturing complex, long-range dependencies among time-varying confounders. Sun et. al.~\cite{sun2023CPT}  proposed  Causal Trajectory Prediction (CTP) model for predicting non-communicable disease progression combines trajectory prediction with causal discovery, enhancing interpretability and aiding clinical decision-making by estimating treatment effect bounds.  Cao et al.~\cite{cao2023estimating} proposed LipCDE, an innovative approach for estimating individualized treatment effects by leveraging Lipschitz regularization and neural controlled differential equations to model dynamic causal relationships and reduce bias in time series data with latent confounders.  Frauen et al.~\cite{frauen2023estimating} developed DeepACE, an end-to-end deep learning model designed for estimating time-varying average causal effects from observational data, addressing latent confounders.  However, these methods are not applicable when there exist latent time-dependent confounders between the single treatment and outcome, as TSD and LipCDE are designed for multiple treatments, while DeepACE and CT assume the absence of latent confounders. Different from these  works, TIFM focuses on the latent time-dependent confounder between the single treatment and outcome and is an IV-based  causal effect estimator. 

\noindent \textbf{IV-based methods for causal effect estimations over time}. 
Recently, there have been some works on developing IV-based methods for analysing time-to-event outcomes in the presence of latent time-dependent confounders. For instance, Martinussen et al.~\cite{martinussen2017instrumental} developed a novel IV estimator via a semiparametric structural cumulative model in the context of time-to-event outcomes.
This seminal contribution takes shape in the form of a structural accelerated failure model tailored for time-event analysis. In a more contemporary stride, Michael et al. \cite{michael2023instrumental} delve into the realm of identifying and estimating MSMMS. They navigate the landscape of time-dependent treatments and introduce the novel concept of time-dependent IV into the mix.  Further expanding the frontiers, Cui et al. \cite{cui2023instrumental} delve into the realm of sufficient conditions for identifying parameters within a marginal structural model using temporal data. Their methodology hinges on the utilisation of time-dependent IV, which is a valuable resource derived from domain knowledge or expertise. Rather than being provided with a pre-existing time-dependent IV, our TIFM aims to learn and generate a substitute of the latent time-dependent IV using history data.

In comparison to the reviewed works, our study focuses on addressing causal effect estimation over time in the presence of latent time-dependent confounders. We achieve this by generating a substitute for the latent time-dependent IV with minimal reliance on domain knowledge. Through our research, we aim to contribute to the advancement of methodologies for addressing the complexities inherent in longitudinal data analysis.

\section{Conclusion}
In this paper, we propose a novel sequential Time-dependent Instrumental Factor Model (TIFM) designed for learning a substitute of the latent time-dependent IV. This substitution enables accurate estimation of causal effects in longitudinal data, particularly when facing time-dependent latent confounders.
We provide theoretical evidence to establish the validity of TIFM in the context of learning the substitute for the latent time-dependent IV from data. Additionally, we devise a LSTM architecture to effectively capture the surrogate of the latent time-dependent IV in practical scenarios. To evaluate the performance of TIFM in causal effect estimation over time while accounting for latent time-dependent confounders, we conduct extensive experiments on synthetic datasets. To further illustrate the applicability of TIFM, we conduct a case study using a real-world climate dataset. The results of this case study underscore the potential of TIFM in real-world applications.

\section{Acknowledgments}
We wish to acknowledge the support from the Australian Research Council (under grant DP230101122).  Thuc Duy Le is supported by DECRA (DE200100200).

\bibliography{tifm2024}

\begin{thebibliography}{48}
\providecommand{\natexlab}[1]{#1}

\bibitem[{Abadi, Agarwal et~al.(2015)}]{tensorflow2015}
Abadi, M.; Agarwal, A.; et~al. 2015.
\newblock {TensorFlow}: large-scale machine learning on heterogeneous systems.
\newblock Software available from tensorflow.org.

\bibitem[{Ali, Groenwold et~al.(2016)}]{ali2016methodological}
Ali, M.~S.; Groenwold, R.~H.; et~al. 2016.
\newblock Methodological comparison of marginal structural model, time-varying
  Cox regression, and propensity score methods: the example of antidepressant
  use and the risk of hip fracture.
\newblock \emph{Pharmacoepidemiology and Drug Safety}, 25: 114--121.

\bibitem[{Angrist and Imbens(1995)}]{angrist1995two}
Angrist, J.~D.; and Imbens, G.~W. 1995.
\newblock Two-stage least squares estimation of average causal effects in
  models with variable treatment intensity.
\newblock \emph{Journal of the American Statistical Association}, 90(430):
  431--442.

\bibitem[{Athey, Tibshirani, and Wager(2019)}]{athey2019generalized}
Athey, S.; Tibshirani, J.; and Wager, S. 2019.
\newblock Generalized random forests.
\newblock \emph{The Annals of Statistics}, 47(2): 1148--1178.

\bibitem[{Bica, Alaa et~al.(2020)}]{bica2020time}
Bica, I.; Alaa, A.; et~al. 2020.
\newblock Time series deconfounder: Estimating treatment effects over time in
  the presence of hidden confounders.
\newblock In \emph{International Conference on Machine Learning}, 884--895.
  PMLR.

\bibitem[{Bica, Alaa et~al.(2019)}]{bica2019estimating}
Bica, I.; Alaa, A.~M.; et~al. 2019.
\newblock Estimating counterfactual treatment outcomes over time through
  adversarially balanced representations.
\newblock In \emph{International Conference on Learning Representations},
  1--28.

\bibitem[{Bowden and Turkington(1990)}]{bowden1990instrumental}
Bowden, R.~J.; and Turkington, D.~A. 1990.
\newblock \emph{Instrumental variables}.
\newblock 8. Cambridge university press.

\bibitem[{Cao, Enouen et~al.(2023)}]{cao2023estimating}
Cao, D.; Enouen, J.; et~al. 2023.
\newblock Estimating Treatment Effects from Irregular Time Series Observations
  with Hidden Confounders.
\newblock In \emph{Proceedings of the AAAI Conference on Artificial
  Intelligence}, volume~37, 6897--6905.

\bibitem[{Cheng, Li et~al.(2022)}]{cheng2022data}
Cheng, D.; Li, J.; et~al. 2022.
\newblock Data-driven causal effect estimation based on graphical causal
  modelling: A survey.
\newblock \emph{ACM Computing Surveys}, 1--35.

\bibitem[{Cheng, Xu et~al.(2023)}]{cheng2023causal}
Cheng, D.; Xu, Z.; et~al. 2023.
\newblock Causal inference with conditional instruments using deep generative
  models.
\newblock In \emph{Proceedings of the AAAI Conference on Artificial
  Intelligence}, volume~37, 7122--7130.

\bibitem[{Chernozhukov, Chetverikov et~al.(2018)}]{chernozhukov2018double}
Chernozhukov, V.; Chetverikov, D.; et~al. 2018.
\newblock Double/debiased machine learning for treatment and structural
  parameters.
\newblock \emph{The Econometrics Journal}, 21(1): C1--C68.

\bibitem[{Cui et~al.(2023)Cui, Michael, Tanser, and
  Tchetgen~Tchetgen}]{cui2023instrumental}
Cui, Y.; Michael, H.; Tanser, F.; and Tchetgen~Tchetgen, E. 2023.
\newblock Instrumental variable estimation of the marginal structural Cox model
  for time-varying treatments.
\newblock \emph{Biometrika}, 110(1): 101--118.

\bibitem[{Frauen et~al.(2023)Frauen, Hatt, Melnychuk, and
  Feuerriegel}]{frauen2023estimating}
Frauen, D.; Hatt, T.; Melnychuk, V.; and Feuerriegel, S. 2023.
\newblock Estimating average causal effects from patient trajectories.
\newblock In \emph{Proceedings of the AAAI Conference on Artificial
  Intelligence}, volume~37, 7586--7594.

\bibitem[{Hartford, Lewis et~al.(2017)}]{hartford2017deep}
Hartford, J.; Lewis, G.; et~al. 2017.
\newblock Deep {IV}: {A} flexible approach for counterfactual prediction.
\newblock In \emph{International Conference on Machine Learning}, 1414--1423.

\bibitem[{Haylock, Hofstra et~al.(2008)}]{haylock2008european}
Haylock, M.~R.; Hofstra, N.; et~al. 2008.
\newblock A European daily high-resolution gridded data set of surface
  temperature and precipitation for 1950–2006.
\newblock \emph{Journal of Geophysical Research: Atmospheres}, 113(D20).

\bibitem[{Hern{\'a}n, Brumback, and Robins(2001)}]{hernan2001marginal}
Hern{\'a}n, M.~A.; Brumback, B.; and Robins, J.~M. 2001.
\newblock Marginal structural models to estimate the joint causal effect of
  nonrandomized treatments.
\newblock \emph{Journal of the American Statistical Association}, 96(454):
  440--448.

\bibitem[{Hern{\'a}n and Robins(2006)}]{hernan2006instruments}
Hern{\'a}n, M.~A.; and Robins, J.~M. 2006.
\newblock Instruments for causal inference: an epidemiologist's dream?
\newblock \emph{Epidemiology}, 360--372.

\bibitem[{Hern{\'a}n and Robins(2010)}]{hernan2020causal}
Hern{\'a}n, M.~A.; and Robins, J.~M. 2010.
\newblock \emph{Causal {I}nference}.
\newblock CRC Boca Raton, FL;.

\bibitem[{Imbens(2014)}]{imbens2014instrumental}
Imbens, G.~W. 2014.
\newblock Instrumental variables: an econometrician’s perspective.
\newblock \emph{Statistical Science}, 29(3): 323--358.

\bibitem[{Imbens and Rubin(2015)}]{imbens2015causal}
Imbens, G.~W.; and Rubin, D.~B. 2015.
\newblock \emph{Causal inference in statistics, social, and biomedical
  sciences}.
\newblock Cambridge University Press.

\bibitem[{Kallenberg and Kallenberg(1997)}]{kallenberg1997foundations}
Kallenberg, O.; and Kallenberg, O. 1997.
\newblock \emph{Foundations of modern probability}, volume~2.
\newblock Springer.

\bibitem[{Kalnay, Kanamitsu et~al.(1996)}]{kalnay1996ncep}
Kalnay, E.; Kanamitsu, M.; et~al. 1996.
\newblock The NCEP/NCAR 40-year reanalysis project.
\newblock \emph{Bulletin of the American Mseteorological Society}, 77(3):
  437--471.

\bibitem[{Keith~Battocchi et~al.(2019)}]{econml}
Keith~Battocchi, E.~D.; et~al. 2019.
\newblock {EconML}: {A Python Package for ML-Based Heterogeneous Treatment
  Effects Estimation}.
\newblock https://github.com/py-why/EconML.
\newblock Version 0.14.

\bibitem[{K{\"u}nzel, Sekhon et~al.(2019)}]{kunzel2019metalearners}
K{\"u}nzel, S.~R.; Sekhon, J.~S.; et~al. 2019.
\newblock Metalearners for estimating heterogeneous treatment effects using
  machine learning.
\newblock \emph{PNAS}, 116(10): 4156--4165.

\bibitem[{Lim, Alaa, and Schaar(2018)}]{lim2018forecasting}
Lim, B.; Alaa, A.; and Schaar, M. v.~d. 2018.
\newblock Forecasting treatment responses over time using recurrent marginal
  structural networks.
\newblock In \emph{Proceedings of the 32nd International Conference on Neural
  Information Processing Systems}, 7494--7504.

\bibitem[{Martens, Pestman et~al.(2006)}]{martens2006instrumental}
Martens, E.~P.; Pestman, W.~R.; et~al. 2006.
\newblock Instrumental variables: application and limitations.
\newblock \emph{Epidemiology}, 260--267.

\bibitem[{Martinussen, Vansteelandt et~al.(2017)}]{martinussen2017instrumental}
Martinussen, T.; Vansteelandt, S.; et~al. 2017.
\newblock Instrumental variables estimation of exposure effects on a
  time-to-event endpoint using structural cumulative survival models.
\newblock \emph{Biometrics}, 73(4): 1140--1149.

\bibitem[{Mastakouri, Sch{\"o}lkopf, and
  Janzing(2021)}]{mastakouri2021necessary}
Mastakouri, A.~A.; Sch{\"o}lkopf, B.; and Janzing, D. 2021.
\newblock Necessary and sufficient conditions for causal feature selection in
  time series with latent common causes.
\newblock In \emph{International Conference on Machine Learning}, 7502--7511.
  PMLR.

\bibitem[{Melnychuk, Frauen, and Feuerriegel(2022)}]{melnychuk2022causal}
Melnychuk, V.; Frauen, D.; and Feuerriegel, S. 2022.
\newblock Causal transformer for estimating counterfactual outcomes.
\newblock In \emph{International Conference on Machine Learning}, 15293--15329.
  PMLR.

\bibitem[{Michael, Cui et~al.(2023)}]{michael2023instrumental}
Michael, H.; Cui, Y.; et~al. 2023.
\newblock Instrumental variable estimation of marginal structural mean models
  for time-varying treatment.
\newblock \emph{Journal of the American Statistical Association}, 1--12.

\bibitem[{Nie and Wager(2021)}]{nie2021quasi}
Nie, X.; and Wager, S. 2021.
\newblock Quasi-oracle estimation of heterogeneous treatment effects.
\newblock \emph{Biometrika}, 108(2): 299--319.

\bibitem[{Pearl(2009)}]{pearl2009causality}
Pearl, J. 2009.
\newblock \emph{Causality}.
\newblock Cambridge university press.

\bibitem[{Peters, Janzing, and Sch{\"o}lkopf(2017)}]{peters2017elements}
Peters, J.; Janzing, D.; and Sch{\"o}lkopf, B. 2017.
\newblock \emph{Elements of causal inference: foundations and learning
  algorithms}.
\newblock The MIT Press.

\bibitem[{Robins(1986)}]{robins1986new}
Robins, J. 1986.
\newblock A new approach to causal inference in mortality studies with a
  sustained exposure period—application to control of the healthy worker
  survivor effect.
\newblock \emph{Mathematical Modelling}, 7(9-12): 1393--1512.

\bibitem[{Robins(1997)}]{robins1997causal}
Robins, J.~M. 1997.
\newblock Causal inference from complex longitudinal data.
\newblock In \emph{Latent variable modeling and applications to causality},
  69--117. Springer.

\bibitem[{Robins(2000)}]{robins2000marginalbook}
Robins, J.~M. 2000.
\newblock Marginal structural models versus structural nested models as tools
  for causal inference.
\newblock In \emph{Statistical Models in Epidemiology, the Environment, and
  Clinical Trials}, 95--133. Springer.

\bibitem[{Robins, Hernan, and Brumback(2000)}]{robins2000marginal}
Robins, J.~M.; Hernan, M.~A.; and Brumback, B. 2000.
\newblock Marginal structural models and causal inference in epidemiology.
\newblock \emph{Epidemiology}, 550--560.

\bibitem[{Robins and Tsiatis(1991)}]{robins1991correcting}
Robins, J.~M.; and Tsiatis, A.~A. 1991.
\newblock Correcting for non-compliance in randomized trials using rank
  preserving structural failure time models.
\newblock \emph{Communications in Statistics-Theory and Methods}, 20(8):
  2609--2631.

\bibitem[{Runge et~al.(2019)Runge, Bathiany, Bollt, Camps-Valls, Coumou, Deyle,
  Glymour, Kretschmer, Mahecha, Mu{\~n}oz-Mar{\'\i}
  et~al.}]{runge2019inferring}
Runge, J.; Bathiany, S.; Bollt, E.; Camps-Valls, G.; Coumou, D.; Deyle, E.;
  Glymour, C.; Kretschmer, M.; Mahecha, M.~D.; Mu{\~n}oz-Mar{\'\i}, J.; et~al.
  2019.
\newblock Inferring causation from time series in Earth system sciences.
\newblock \emph{Nature Communications}, 10(1): 2553.

\bibitem[{Seeley and Romps(2020)}]{seeleyromps2020}
Seeley, J.~T.; and Romps, D.~M. 2020.
\newblock Temperature and CAPE dependence in observed and modelled convective
  precipitation over the United States.
\newblock \emph{Journal of Climate}, 33(17): 7305--7322.

\bibitem[{Semenova et~al.(2023)Semenova, Goldman, Chernozhukov, and
  Taddy}]{Semenova2023Inference}
Semenova, V.; Goldman, M.; Chernozhukov, V.; and Taddy, M. 2023.
\newblock Inference on heterogeneous treatment effects in high-dimensional
  dynamic panels under weak dependence.
\newblock \emph{Quantitative Economics}, 14(2): 471--510.

\bibitem[{Spirtes, Glymour et~al.(2000)}]{spirtes2000causation}
Spirtes, P.; Glymour, C.~N.; et~al. 2000.
\newblock \emph{Causation, prediction, and search}.
\newblock MIT press.

\bibitem[{Sun, Zhang et~al.(2023)}]{sun2023CPT}
Sun, Z.; Zhang, W.; et~al. 2023.
\newblock CTP: A Causal Interpretable Model for Non-Communicable Disease
  Progression Prediction.
\newblock \emph{arxiv.org/abs/2308.09735}.

\bibitem[{Syrgkanis, Lei et~al.(2019)}]{syrgkanis2019machine}
Syrgkanis, V.; Lei, V.; et~al. 2019.
\newblock Machine learning estimation of heterogeneous treatment effects with
  instruments.
\newblock In \emph{International Conference on Neural Information Processing
  Systems}, 15193--15202.

\bibitem[{Uppala, Kållberg et~al.(2005)}]{uppala2005era}
Uppala, S.~M.; Kållberg, P.~W.; et~al. 2005.
\newblock The ERA-40 re-analysis.
\newblock \emph{Quarterly Journal of the Royal Meteorological Society: A
  Journal of the Atmospheric Sciences, Applied Meteorology and Physical
  Oceanography}, 131(612): 2961--3012.

\bibitem[{Wang, Yang et~al.(2022)}]{wang2022estimating}
Wang, H.; Yang, W.; et~al. 2022.
\newblock Estimating individualized causal effect with confounded instruments.
\newblock In \emph{Proceedings of the 28th ACM SIGKDD Conference on Knowledge
  Discovery and Data Mining}, 1857--1867.

\bibitem[{Wang and Blei(2019)}]{wang2019blessings}
Wang, Y.; and Blei, D.~M. 2019.
\newblock The blessings of multiple causes.
\newblock \emph{Journal of the American Statistical Association}, 114(528):
  1574--1596.

\bibitem[{Zolina, Simmer et~al.(2013)}]{zolina2013changes}
Zolina, O.; Simmer, C.; et~al. 2013.
\newblock Changes in the duration of European wet and dry spells during the
  last 60 years.
\newblock \emph{Journal of Climate}, 26(6): 2022--2047.

\end{thebibliography}

\newpage
\appendix		
\textbf*{Supplementary Material for ``Instrumental Variable Estimation for Causal Inference in Longitudinal Data with Time-Dependent Latent Confounders''}

The supplement is provided to the paper ``Instrumental Variable Estimation for Causal Inference in Longitudinal Data with Time-Dependent Latent Confounders''. We present two lemmas for the proof of Theorem 1 in the main text, , the implementations of compared models, parameters setting,  details of synthetic data generation and more experimental results.

\section{The Proposed TIFM Method}
\subsection{Obejective Formulation}
In our main text, we introduce the Sequential Kallenberg Construction, a modified definition of the ``Kallenburg Construction'' in~\cite{kallenberg1997foundations,bica2020time}.

\begin{definition} [Sequential Kallenberg construction]
	At time-step $t$, the distribution of $\mathbf{X}_t = [\mathbf{X}_{t1}, \dots, \mathbf{X}_{tk}]$ follows a sequential Kallenberg construction through the random variables $\mathbf{l}_t = f(\bar{\mathbf{h}}_{t-1})$, provided that there exist measurable functions $f_{tj}: \mathcal{L} \times [0, 1] \to \mathcal{X}_t$ and random variables $M_{tj} \in [0, 1]$, where $j = 1, \dots, k$, satisfy the condition $\mathbf{X}_{tj} = f_{tj}(\mathbf{L}_t, M_{tj})$, with $M_{tj} \in [0, 1]$ jointly satisfying $(M_{t1}, \dots, M_{tk})\ci \bar{W}_t\mid \bar{\mathbf{X}}_{t-1},\bar{\mathbf{L}}_t$.
\end{definition}

We then present two lemmas that are adapted from the work in~\cite{bica2020time}, which serve to prove our Theorem 1.

\begin{lemma} \label{lemma:001} 
	If at every timestep $t$, the marginal distribution of  $(\mathbf{X}_{t1}, \dots, \mathbf{X}_{tk})$ admits a Kallenberg construction from $\mathbf{l}_t=f(\bar{\mathbf{h}}_{t-1})$, then we obtain $(\mathbf{X}_{t1}, \dots, \mathbf{X}_{tk})\ci \bar{W}_t\mid \bar{\mathbf{X}}_{t-1},\bar{\mathbf{L}}_t$.
\end{lemma}
\begin{proof}
	Assume that $\mathbf{X}_j$ for $j=1, \dots, k$ are Borel spaces. $\forall t \in \{1, \dots, T\}$ assume $\mathcal{L}_t$ is a measured space, and assume that $\mathbf{X}_{tj} = f_{tj}(\mathbf{L}_t, M_{tj})$.  Hence, we have $(M_{t1}, \dots, M_{tk})\ci \bar{W}_t\mid \bar{\mathbf{X}}_{t-1},\bar{\mathbf{L}}_t$. This implies that $(\bar{\mathbf{L}}_t, M_{t1}, \dots, M_{tk})\ci \bar{W}_t\mid \bar{\mathbf{X}}_{t-1},\bar{\mathbf{L}}_t$. Since $\mathbf{X}_{tj}$'s are measured functions of  $(\mathbf{L}_t, M_{tj})$, and $\bar{\mathbf{H}}_{t-1}=(\bar{W}_{t-1}, \bar{\mathbf{X}}_{t-1}, \bar{\mathbf{L}}_{t-1})$,  we have $(\mathbf{X}_{t1}, \dots, \mathbf{X}_{tk})\ci \bar{W}_t\mid \bar{\mathbf{X}}_{t-1},\bar{\mathbf{L}}_t$, for $\forall t\in\{1, \dots, T\}$. 
\end{proof}

Similar to Lemma 2 in~\cite{bica2020time}, we present the following lemma to ensure that the factor modes for $\bar{\mathbf{X}}_t$ satisfy the Sequential Kallenberg construction at each timestep $t$.

\begin{lemma} 
	\label{lemma:002} 
	Under weak regularity conditions, if the distribution $p(\bar{\mathbf{x}}_T)$  can be represented using the factor model $p(\sigma_{1:k}, \bar{\mathbf{x}}_T, \bar{\mathbf{l}}_T)$, then we obtain a sequential Kallenberg construction at each timestep. 
\end{lemma}
\begin{proof}
	First, without loss of generality,   the product space of each domain  $\mathcal{L}$ is assumed to be a Borel space, which can be reduced to $\mathcal{L}=[0,1]^k$~\cite{kallenberg1997foundations}.  There exists some measurable function $h_{ij}$ such that ${M}_{ij} = h_{ij}(\sigma_i, \omega_{tj})$, where $\omega_{tj}\sim Uniform[0, 1]$ and $j =1, \dots, k$.  Similar to Lemma 2 in~\cite{bica2020time}, $\sigma_{1:k}$ are parameters in the factor model and can be considered point masses, so we have $(\sigma_1,  \dots, \sigma_{k})\ci \bar{W}_t\mid \bar{\mathbf{X}}_{t-1},\bar{\mathbf{L}}_t$. Because ${M}_{ij} = h_{ij}(\sigma_i, \omega_{tj})$ are measureable function of $\sigma_i$ and $\omega_{tj}$, we have $(M_{t1}, \dots, M_{tk})\ci \bar{W}_t\mid \bar{\mathbf{X}}_{t-1},\bar{\mathbf{L}}_t$. 
	
	Therefore, we have a sequential Kallenberg construction at timestep $t$. 
\end{proof}

%\begin{theorem}
\noindent\textbf{Theorem 1}
	\emph{If the distribution $p(\bar{\mathbf{x}}_T)$ can be represented using the factor model $p(\sigma_{1:k}, \bar{\mathbf{x}}_T, \bar{\mathbf{l}}_T)$, we can deduce that $\bar{S}_t$ is captured by the substitute $\bar{\mathbf{L}}_t$ which serves as a time-dependent IV.}
%\end{theorem}
\begin{proof}
	First, we proof that $\bar{\mathbf{L}}_t$ is a substitute of $\bar{S}_t$. Based on Lemma~\ref{lemma:001}, if at every time-step $t$, the distribution of $\mathbf{X}_t = (\mathbf{X}_{t1}, \dots, \mathbf{X}_{tk})$ admits the Kallenberg construction through $\bar{\mathbf{L}}_t = f(\bar{\mathbf{H}}_{t-1})$ satisfying $(M_{t1}, \dots, M_{tk})\ci \bar{W}_t\mid \bar{\mathbf{X}}_{t-1},\bar{\mathbf{L}}_t$, then we have  $(\mathbf{X}_{t1}, \dots, \mathbf{X}_{tk})\ci \bar{W}_t\mid \bar{\mathbf{X}}_{t-1},\bar{\mathbf{L}}_t$ (More details see the supplement.). Furthermore, according to Lemma~\ref{lemma:002},  if the distribution $p(\bar{\mathbf{x}}_T)$ can be represented using the factor model $p(\sigma_{1:k}, \bar{\mathbf{x}}_T, \bar{\mathbf{l}}_T)$, then we have a sequential Kallenberg construction for each time-step. Hence, $\bar{S}_t$ must be captured by the substitute $\bar{\mathbf{L}}_t$.
	
	Next, $\bar{\mathbf{L}}_t$ is a substitute of the latent time-dependent IV $\bar{S}_t$, so at the time-step $t$, we have that (i). $\mathbb{E}(W_t\mid \bar{W}_{t-1}, \bar{\mathbf{X}}_t, \bar{\mathbf{L}}_t)\neq\mathbb{E}(W_t\mid \bar{W}_{t-1}, \bar{\mathbf{X}}_t, \bar{\mathbf{L}}_{t-1})$ (Assumptions~3 holds); (ii). $\bar{\mathbf{L}}_t\indep(Y_{t+1} (\bar{w}), \mathbf{X}_{t+1}, \mathbf{U}_{t+1})\mid \bar{W}_t =\bar{w}_t, \bar{\mathbf{X}}_t, \bar{\mathbf{U}}_t$ (Assumptions~4 holds); and (iii). $\bar{\mathbf{L}}_t\indep\bar{\mathbf{U}}\mid \bar{W}_{t-1}, \bar{\mathbf{X}}_t, \bar{\mathbf{L}}_{t-1}$ (Assumptions~5 holds). Therefore, the substitute $\bar{\mathbf{L}}_t$ serves as a time-dependent IV.
\end{proof}

\section{Experiments}
\subsection{Experiment Setup}
\paragraph{Data Generation} To ensure that the synthetic datasets closely resemble real-world scenarios, we adopt the same procedure as outlined in~\cite{bica2020time}, wherein $p$-order autoregressive processes are used to generate the datasets. At each time-step $t$, we simulate the time-dependent covariates $\mathbf{X}_{t}$, latent time-dependent confounders $\mathbf{U}_{t}$, and a latent time-dependent  IV $S_{t}$ as follows:
\begin{equation*}
	\begin{aligned}
		&\mathbf{X}_{t} = \frac{1}{p} \sum_{i = 1}^{p} (\alpha_i\mathbf{X}_{t-i} + \omega_i{W}_{t-i}) + \varepsilon_{\mathbf{X}},\\
		&\mathbf{U}_{t} = \frac{1}{p} \sum_{i = 1}^{p} (\beta_i\mathbf{U}_{t-i} + \lambda_i{W}_{t-i}) + \varepsilon_{\mathbf{U}},\\
		&{S}_{t} = \frac{1}{p} \sum_{i = 1}^{p} (\gamma_i{IV}_{t-i} + \delta_i\mathbf{X}_{t-i}) + \varepsilon_{S},
	\end{aligned}
\end{equation*}where $\alpha_i, \lambda_i, \delta_i \sim \mathcal{N}(0,0.5^{2})$, $\omega_i, \beta_i, \gamma_i \sim \mathcal{N}(1-(i/p),(i/p)^{2})$, and $\varepsilon_{\mathbf{X}}, \varepsilon_{\mathbf{U}}, \varepsilon_{S} \sim \mathcal{N}(0,0.01^{2})$. The treatment $W_t$ depends on latent IV $S_t$, latent confounders $\mathbf{U}_t$ and covariates $\mathbf{X}_t$:
\begin{equation*}
	\begin{aligned}
		&\theta_t = \mu_{\mathbf{X}}\widehat{\mathbf{X}_t} + \mu_{\mathbf{U}}\widehat{\mathbf{U}_t} + \mu_{{S}}\widehat{{S}_t},\\
		&W_t | \theta_t \sim Bernoulli(\sigma(c~\theta_t)),
	\end{aligned}
\end{equation*}
\noindent where $\widehat{\mathbf{X}_t}, \widehat{\mathbf{U}_t}, \widehat{{S}_t}$ are the sum of the covariates, latent confounders and latent IV respectively over the last $p$ time steps, $\sigma(\cdot)$ is the sigmoid function, $\mu_{\mathbf{X}_t}, \mu_{\mathbf{U}_t}, \mu_{{IV}_t}, c \sim \mathcal{N}(0,1^{2})$. The outcome $Y_{t+1}$ is obtained as a function of treatment $W_t$, covariates $\mathbf{X}_t$ and latent confounders $\mathbf{U}_t$:
\begin{equation*}
	\begin{aligned}
		Y_{t + 1} = \rho_{{W}}W_t + \rho_{\mathbf{X}}\mathbf{X}_{t} + \rho_{\mathbf{U}}\mathbf{U}_{t}, 
	\end{aligned}
\end{equation*}where $\rho_{{W}}, \rho_{\mathbf{X}},\rho_{\mathbf{U}} = 0.5$. 

In this paper, we generate the synthetic datasets with a variety range of sample sizes (N), i.e.,  2k, 4k, 6k, and 8k. To avoid the bias brought by the data generation process, we repeatedly generated 30 datasets for each sample size. To induce time dependencies, we set $p = 1$ and $p = 3$. The dimensionality of the covariates $\mathbf{X}$ and the latent confounders are set to 3, respectively. 

\paragraph{Parameters Settings} The implementations of all comparisons are from the \emph{Python} package~\cite{econml} with default .   For our TIFM,   the parameter settings of TIFM  are summarised in Table~\ref{tab:setting}.

\begin{table}[t]
	\centering
	\caption{Details of the parameter settings used in TIFM.}
	\label{tab:setting}
	\begin{tabular}{|cc|cc|}
		\toprule
		Parameter & Value & Parameter & Value  \\ \midrule
		Reps & 30 & RNN hidden units & 128  \\
		Epoch & 100 & FC hidden units & 128 \\
		Batch\_Size & 128 & Dropout probability & 0.8  \\ \bottomrule
	\end{tabular}
\end{table}

\begin{figure*}[t]
	\centering
	\includegraphics[scale=0.59]{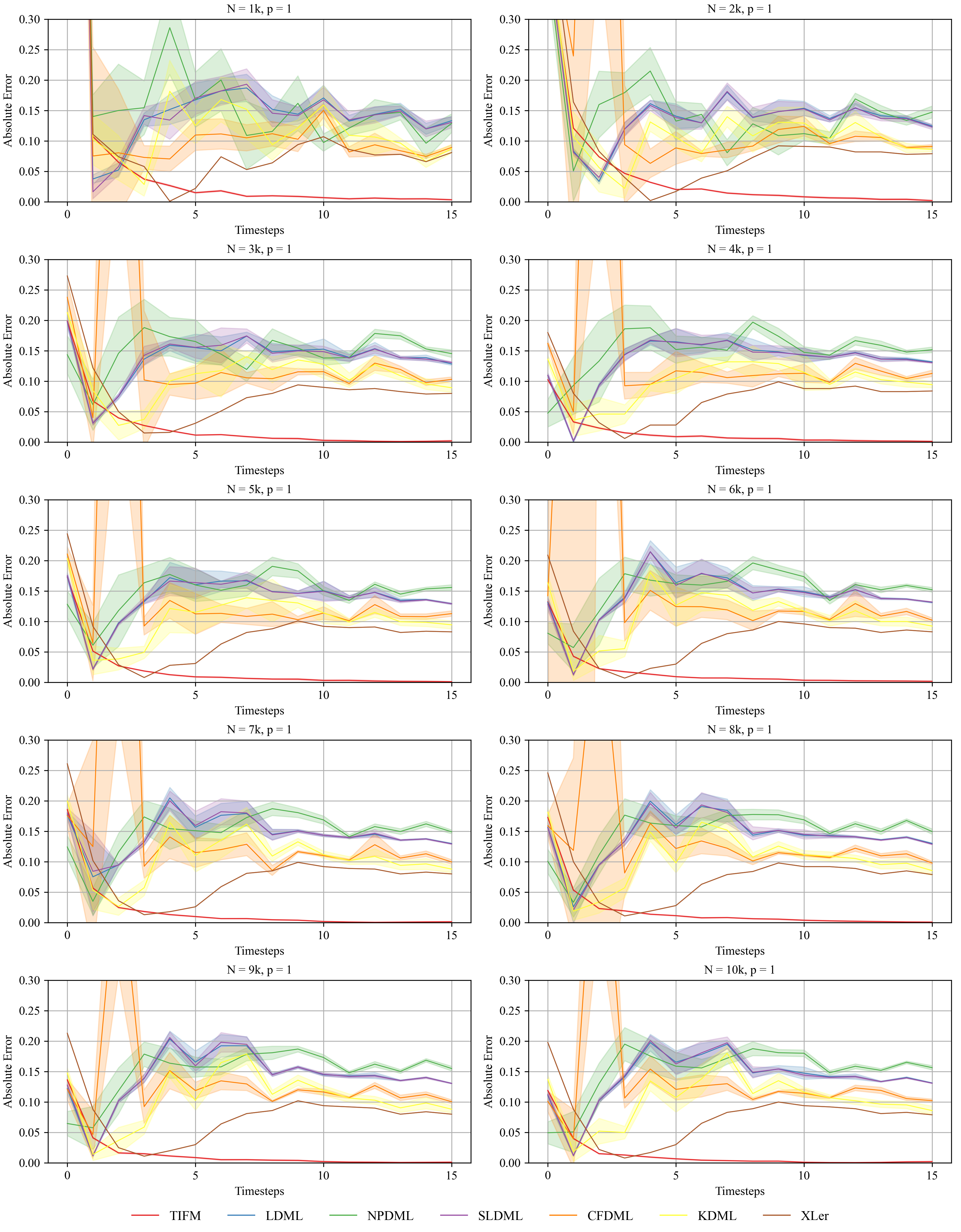}
	\caption{Results of the absolute errors across all estimators, accompanied by the mean plus standard deviation calculated over 30 synthetic datasets. $p$ = 1, where $p$ is from $p$-order autoregressive processes.} \label{pic:p1} 
\end{figure*}

\begin{figure*}[t]
	\centering
	\includegraphics[scale=0.59]{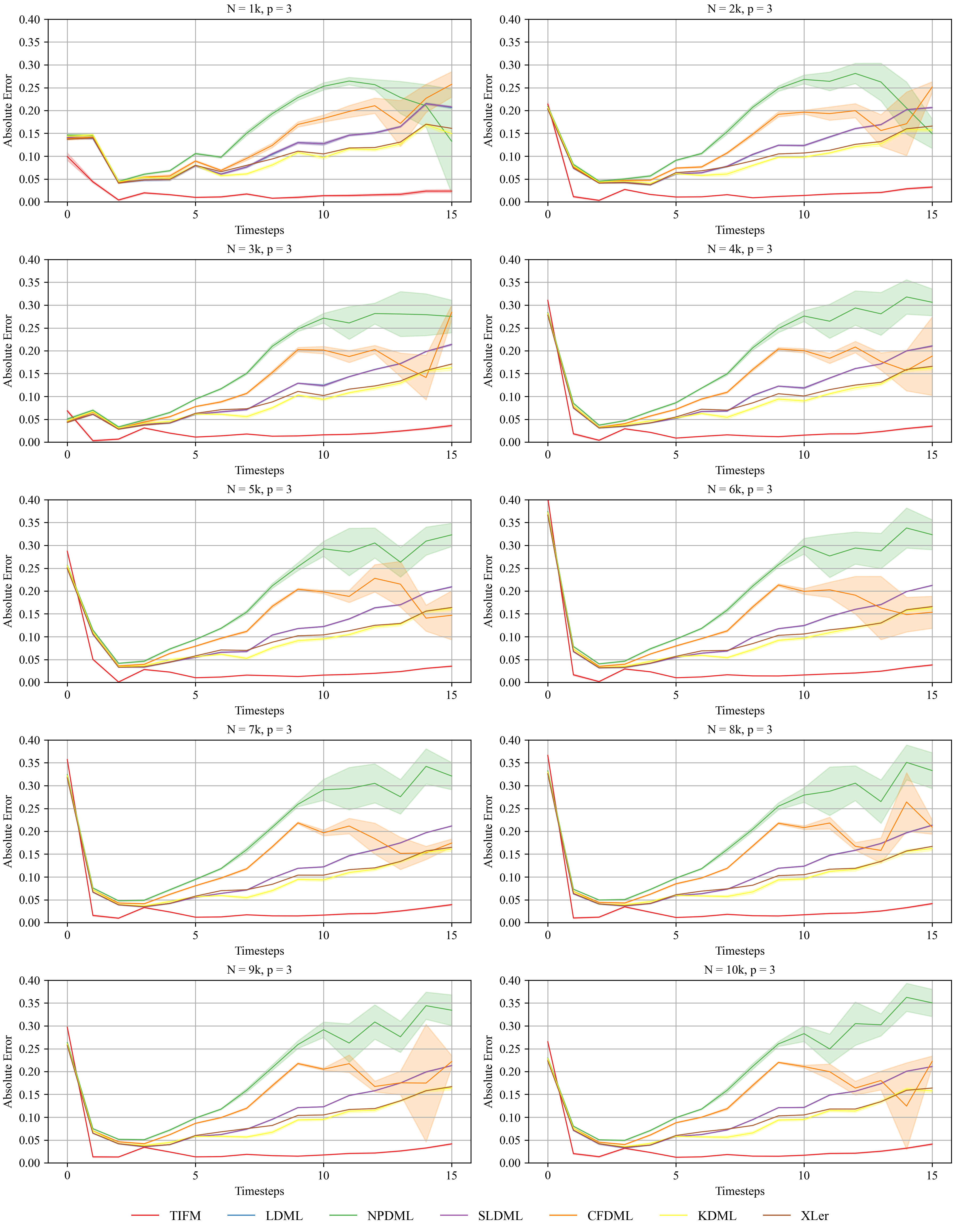}
	\caption{Results of the absolute errors across all estimators, accompanied by the mean plus standard deviation calculated over 30 synthetic datasets. $p$ = 3, where $p$ is from $p$-order autoregressive processes.} \label{pic:p3} \end{figure*}

\paragraph{Performance Evaluation}
We report the experimental results of both double-robust methods, including double robust learning (DRL) and kernel double robust learning (KDRL), along with all the methods used in the main text, across 30 synthetic datasets with 6k, as shown in Table~\ref{tab:compare}. From Table~\ref{tab:compare}, we have the same conclusions as described in the main text. Note that both DRL and KDRL exhibit a significantly large standard deviation, rendering both methods unsuitable for visualisation in our plots.

\begin{table}[t]
	\centering
	\footnotesize
	\caption{Results of the absolute errors across all estimators over 30 synthetic datasets. $p$ = 1, where $p$ is from $p$-order autoregressive processes.}
	\begin{tabular}{cccc}
		\toprule
		6k       & time-step-5        & time-step-10       & time-step-15       \\ \midrule
		Baseline & 0.283         & 0.251         & 0.231         \\ \midrule
		LDML     & 0.172±0.025 & 0.146±0.004 & 0.136±0.001 \\
		NPDML    & 0.177±0.028 & 0.183±0.012 & 0.154±0.003 \\
		SLDML    & 0.166±0.024 & 0.146±0.004 & 0.136±0.001 \\
		CFDML    & 0.135±0.029 & 0.103±0.005 & 0.108±0.005 \\
		KDML     & 0.121±0.040 & 0.130±0.007  & 0.099±0.007 \\
		Xlearner & 0.028±0.000 & 0.101±0.000 & 0.084±0.000 \\ \midrule
		DRL      & 0.256±0.295 & 0.245±0.088 & 0.225±0.038 \\
		FDRL     & 0.278±0.329 & 0.220±0.024  & 0.178±0.014 \\ \midrule
		TIFM     & 0.013±0.001 & 0.005±0.001 & 0.001±0.001 \\ \bottomrule
	\end{tabular}
	\label{tab:compare}
\end{table}

In this supplement, we provide additional comprehensive results for various sample sizes: 1k, 2k, 3k, 4k, 5k, 6k, 7k, 8k, 9k, and 10k, while varying $p=1$ and $p=3$.  The absolute errors across all estimators for datasets with  $p=1$ and $p=3$ are visualised in Fig.~\ref{pic:p1} and~\ref{pic:p3}, respectively.
From Fig.\ref{pic:p1} and~\ref{pic:p3}, we arrive at the same conclusions as described in the main text, i.e., our TIFM method exhibits the lowest absolute error compared to other methods when the timestep is larger than 5.

\end{document}